\ifdefined\iccvformat
  \documentclass[10pt,twocolumn,letterpaper]{article}
  \usepackage{iccv}
  \usepackage[pagebackref=true,breaklinks=true,letterpaper=true,colorlinks,bookmarks=false]{hyperref}
   \iccvfinalcopy 
  \def\iccvPaperID{1082} 
  
\else
	\documentclass[10pt]{article}
  \makeatletter
  \DeclareRobustCommand\onedot{\futurelet\@let@token\@onedot}
  \def\@onedot{\ifx\@let@token.\else.\null\fi\xspace}

\fi

\usepackage{times}
\usepackage{epsfig}
\usepackage{graphicx}
\usepackage{amsmath}
\usepackage{amssymb}

\usepackage{lwm}
\usepackage{multirow}
\usepackage{url}
\usepackage[numbers]{natbib}
\usepackage{algorithm}
\usepackage{algorithmic}
\usepackage{epsfig}
\usepackage{verbatim}
\usepackage{subcaption}
\usepackage{macros}

\ifdefined\iccvformat
\else
\usepackage{ourarxiv,times} 
\nipsfinalcopy

\title{Distributed Low-rank Subspace Segmentation}
\author{
Ameet Talwalkar\textsuperscript{a}\\
\And
Lester Mackey\textsuperscript{b}\\
\And
Yadong Mu\textsuperscript{c}\\
\And
Shih-Fu Chang\textsuperscript{c}\\
\And
Michael I. Jordan\textsuperscript{a, d}\\
}
\fi

\begin{document}

\ifdefined\iccvformat
\title{Distributed Low-rank Subspace Segmentation}

\author{
Ameet Talwalkar\textsuperscript{a} \quad Lester Mackey\textsuperscript{b}
\quad Yadong Mu\textsuperscript{c}
\quad Shih-Fu Chang\textsuperscript{c} \quad Michael I. Jordan\textsuperscript{a}
\\
\noindent\textsuperscript{a}University of California, Berkeley \quad 
\noindent\textsuperscript{b}Stanford University \quad 
\noindent\textsuperscript{c}Columbia University \\
\tt\small{\{ameet, jordan\}@cs.berkeley.edu, lmackey@stanford.edu, \{muyadong, sfchang\}@ee.columbia.edu}
}

\maketitle

\else
\maketitle
{\centering
\vspace*{-0.8cm}
\noindent\textsuperscript{a} Department of Electrical Engineering and Computer Science, UC Berkeley \\
\textsuperscript{b} Department of Statistics, Stanford \\
\textsuperscript{c} Department of Electrical Engineering, Columbia \\
\textsuperscript{d} Department of Statistics, UC Berkeley\par\bigskip
\par
}
\fi

\begin{abstract}

Vision problems ranging from image clustering to motion
segmentation to semi-supervised learning can naturally be framed as
\emph{subspace segmentation} problems, in which one aims to  recover multiple
low-dimensional subspaces from noisy and corrupted input data.
Low-Rank Representation (LRR), a convex formulation of the subspace segmentation
problem, is provably and empirically accurate on small problems but does not
scale to the massive sizes of modern vision datasets.  Moreover, past work
aimed at scaling up low-rank matrix factorization is not applicable to LRR given its
non-decomposable constraints.  In this work, we 
propose a novel divide-and-conquer algorithm for large-scale subspace segmentation
that can cope with LRR's non-decomposable constraints and maintains LRR's strong
recovery guarantees.
This has immediate implications for the scalability of subspace
segmentation, which we demonstrate on a benchmark face recognition dataset and
in simulations.  We then introduce novel applications of LRR-based subspace
segmentation to large-scale semi-supervised learning for multimedia event
detection, concept detection, and image tagging.  In each case, we obtain
state-of-the-art results and order-of-magnitude speed ups.
\end{abstract}

\section{Introduction}

Visual data, though innately high dimensional, often reside in or lie close to
a union of low-dimensional subspaces.  These subspaces might reflect physical
constraints on the objects comprising images and video (e.g., faces under
varying illumination~\cite{BasriJa03} or trajectories of rigid
objects~\cite{TomasiKa92}) or naturally occurring variations in
production (e.g., digits hand-written by different
individuals~\cite{HastieSi98}).  \emph{Subspace segmentation}
techniques model these classes of data by recovering bases for the multiple
underlying subspaces~\cite{Gear1998,ssc}.  Applications
include image clustering~\cite{ssc}, segmentation of images, video, and
motion~\cite{YangWrMaSa08,CosteiraKa98,VidalMaSa05}, and affinity graph construction for
semi-supervised learning~\cite{zhuang12}.

One promising, convex formulation of the subspace segmentation problem is the
\emph{low-rank representation} (LRR) program of Liu et
al.~\cite{LiuLiYu10,LiuXuYa11}:
\begin{align}
\label{eq:lrr_opt}
(\hat{\Z}, \hat{\S}) = 
\argmin_{\Z,\S} \quad \norm{\Z}_*+\lambda\norm{\S}_{2,1} &  \\
\text{subject\, to}\quad \M = \M\Z + \S\,.\nonumber
\end{align}
Here, $\M$ is an input matrix of datapoints drawn from multiple subspaces, 
$\norm{\cdot}_*$ is the nuclear norm,
$\norm{\cdot}_{2,1}$ is the sum of the column $\ell_2$ norms,
and $\lambda$ is a parameter that trades off between these penalties.  LRR segments the columns of $\M$ into
subspaces using the solution $\hat{\Z}$, and, along with its extensions (e.g.,
LatLRR~\cite{LiuYa11} and NNLRS~\cite{zhuang12}), admits strong guarantees of
correctness and strong empirical performance in clustering and graph
construction applications.  However, the standard algorithms for solving
\eqref{eq:lrr_opt} are unsuitable for large-scale problems, due to their
sequential nature and their reliance on the repeated computation of costly
truncated SVDs.

Much of the computational burden in solving LRR stems from the nuclear norm penalty,
which is known to encourage low-rank solutions, so one might hope to leverage the large body
of past work on parallel and distributed matrix
factorization~\cite{GemullaNiHaSi11,RechtRe11,NiuReReWr11,YuHsSiDh12,MackeyTaJo11}
to improve the scalability of LRR.  Unfortunately, these techniques are tailored to 
optimization problems with losses and constraints that decouple
across the entries of the input matrix.  
This decoupling requirement is violated in the LRR problem due to the $\M = \M\Z + \S$
constraint of \eqref{eq:lrr_opt},
and this
non-decomposable constraint introduces new algorithmic and
analytic challenges that do not arise in decomposable matrix factorization problems. 

To address these challenges, we develop, analyze, and evaluate a provably accurate divide-and-conquer approach to large-scale subspace segmentation that specifically accounts for the non-decomposable structure of the LRR problem. Our contributions are three-fold:

\vspace{2mm}
\textbf{Algorithm:} We introduce a parallel, divide-and-conquer approximation algorithm for
LRR that is suitable for large-scale subspace segmentation problems. 
Scalability is achieved by dividing the original LRR
problem into computationally tractable and communication-free subproblems,
solving the subproblems in parallel, and combining the results using a technique
from randomized matrix approximation.
Our algorithm, which we call \fastmf-LRR, is based on the principles of the
Divide-Factor-Combine (\fastmf) framework \cite{MackeyTaJo11} for decomposable
matrix factorization but 
can cope with the non-decomposable
constraints of LRR.    

\textbf{Analysis:} 
We characterize the segmentation behavior of our new algorithm,
showing that \fastmf-LRR maintains the segmentation guarantees 
of the original LRR algorithm with high probability, even while enjoying substantial speed-ups over its namesake.
Our new analysis features a significant broadening of the original
LRR theory to treat the richer class of LRR-type subproblems that arise in \fastmf-LRR.
Moreover, since our ultimate goal is subspace segmentation and not matrix recovery, 
our theory guarantees correctness under a more substantial reduction of problem complexity than the work of~\cite{MackeyTaJo11} (see Sec.~\ref{sec:lrr-theory} for more details).

\textbf{Applications:}  We first present results on face
clustering and synthetic subspace segmentation to demonstrate that \fastmf-LRR
achieves accuracy comparable to LRR in a fraction of the time.  We then
propose and validate a novel application of the LRR
methodology to large-scale graph-based semi-supervised learning. 
While LRR has been used to construct affinity graphs for semi-supervised learning in the past~\cite{chengb11,zhuang12}, 
prior attempts have failed to scale to the sizes of real-world datasets.  
Leveraging the favorable computational properties of
\fastmf-LRR, we propose a scalable strategy for constructing such
subspace affinity graphs. 
We apply our methodology to a variety of computer vision tasks -- multimedia
event detection, concept detection, and image tagging -- demonstrating an order
of magnitude improvement in speed and accuracy that exceeds the state of the art.

The remainder of the paper is organized as follows. In
Section~\ref{sec:fast_algs} we first review the low-rank representation
approach to subspace segmentation and then introduce our novel \fastmf-LRR 
algorithm.  Next, we present our theoretical analysis of \fastmf-LRR in
Section~\ref{sec:theory}. Section~\ref{sec:experiments} highlights the accuracy
and efficiency of \fastmf-LRR on a variety of computer vision tasks. We present
subspace segmentation results on simulated and real-world data in
Section~\ref{ssec:sub_seg}.  In Section~\ref{ssec:semi_super} we present our
novel application of \fastmf-LRR to graph-based semi-supervised learning
problems, and we conclude in Section~\ref{sec:conclusion}.

\textbf{Notation}
Given a matrix $\M \in \reals^{m \times n}$, we define $\U_{M} \mSigma_{M} \V_{M}^\top$ as the compact singular value
decomposition (SVD) of $\M$, where $\rank{\M}=r$, $\mSigma_M$ is a diagonal
matrix of the $r$ non-zero singular values and $\U_M \in \reals^{m \times r}$
and $\V_M \in \reals^{n \times r}$ are the associated left and right singular
vectors of $\M$. We denote the orthogonal projection onto
the column space of $\M$ as $\P_M$.

\section{Divide-and-Conquer Segmentation}
\label{sec:fast_algs}
In this section, we review the LRR approach to subspace segmentation and present our
novel algorithm, \fastmf-LRR.  

\subsection{Subspace Segmentation via LRR}
\label{ssec:subspace_segmentation}
In the \emph{robust subspace segmentation} problem, we observe a matrix $\M =
\mL_0 + \S_0\in\reals^{m\times n}$, where the columns of $\mL_0$ are datapoints
drawn from multiple independent subspaces,\footnote{Subspaces are \emph{independent} if the dimension of their direct sum is the sum of their dimensions.}
and $\S_0$ is a column-sparse outlier matrix.
Our goal is to identify the subspace associated with each column of $\mL_0$,
despite the potentially gross corruption introduced by $\S_0$.  An important
observation for this task is that the projection matrix $\V_{L_0}\V_{L_0}^\top$
for the row space of $\mL_0$, sometimes termed the \emph{shape iteration matrix}, is block diagonal whenever the columns of $\mL_0$
lie in multiple independent subspaces~\cite{Gear1998}.  Hence, we can achieve accurate
segmentation by first recovering the row space of $\mL_0$.

The LRR approach of~\cite{LiuLiYu10} seeks to recover the row space of $\mL_0$ by solving
the convex optimization problem presented in \eqref{eq:lrr_opt}.  Importantly, the LRR
solution comes with a guarantee of correctness: the column space of $\hat{\Z}$
is exactly equal to the row space of $\mL_0$ whenever certain technical
conditions are met~\cite{LiuXuYa11} (see Sec.~\ref{sec:theory} for more details).  

Moreover, as we will show in this work, LRR is also well-suited to the
construction of affinity graphs for semi-supervised learning. 
In this setting, the goal is to define an affinity graph in which nodes
correspond to data points and edge weights exist between nodes drawn from the
same subspace.  LRR can thus be used to recover the block-sparse structure of
the graph's affinity matrix, and these affinities can be used for
semi-supervised label propagation.

\subsection{Divide-Factor-Combine LRR (DFC-LRR)}
We now present our scalable divide-and-conquer algorithm, called \fastmf-LRR,
for LRR-based subspace segmentation.  
\fastmf-LRR extends the principles of the \fastmf framework
of~\cite{MackeyTaJo11} to a new non-decomposable problem.  
The \fastmf-LRR algorithm is summarized in Algorithm
\ref{alg:dfc_lrr}, and we next describe each step in further detail.
\newline\newline\noindent
{\bf D step - Divide input matrix into submatrices:} \fastmf-LRR randomly
partitions the columns of $\M$ into $t$ $l$-column submatrices,
$\{\C_1,\ldots,\C_t\}$. For simplicity, 
we assume that $t$ divides $n$ evenly.
\newline\newline\noindent
{\bf F step - Factor submatrices in parallel:} \fastmf-LRR solves $t$ subproblems
in parallel.  The $i$th LRR subproblem is of the form
\begin{align}
\label{eq:lrr_opt_dfc}
\min_{\Z_i,\S_i} \quad \norm{\Z_i}_*+\lambda\norm{\S_i}_{2,1} & \\
\text{subject\, to}\quad \C_i = \M\Z_i + \S_i\,, \nonumber
\end{align}
where the input matrix $\M$ is used as a dictionary but only a subset of
columns is used as the observations.\footnote{An alternative formulation 
involves replacing both instances of $\M$ with $\C_i$ in 
\eqref{eq:lrr_opt}.  The resulting low-rank estimate $\hat \Z_i$ would have
dimensions $l \times l$, and the C step of \fastmf-LRR would compute a low-rank
approximation on the block-diagonal matrix
diag($\hat \Z_1, \hat \Z_2, \ldots, \hat \Z_t$).}
A typical LRR algorithm can be easily modified to solve \eqref{eq:lrr_opt_dfc} 
and will return a low-rank estimate $\hat\Z_i$ in factored form.
\newline\newline\noindent
{\bf C step - Combine submatrix estimates:} \fastmf-LRR generates a final
approximation $\Zprojmf$ to the low-rank LRR solution $\hat{\Z}$  by
projecting $[\hat\Z_1,\ldots,\hat\Z_t]$ onto the column space of $\hat \Z_1$.
This \emph{column projection} technique is commonly used to produce randomized 
low-rank matrix factorizations~\cite{KumarMoTa09} and was also employed by the \projmf algorithm of \cite{MackeyTaJo11}.
\newline
\newline
\noindent {\bf Runtime:} As noted in~\cite{MackeyTaJo11}, many state-of-the-art
solvers for nuclear-norm regularized problems like \eqref{eq:lrr_opt} have
$\Omega(mnk_M)$ per-iteration time complexity due to the rank-$k_M$ truncated
SVD required on each iteration.  \fastmf-LRR reduces this per-iteration
complexity significantly and requires just O$(mlk_{C_i})$ time for the $i$th
subproblem. Performing the subsequent column projection step is relatively
cheap computationally, since an LRR solver can return its solution in factored
form.  Indeed, if we define $k' \defeq\max_i k_{C_i}$, then the column
projection step of \fastmf-LRR requires only O$(mk'^2+lk'^2)$ time.

\begin{figure}[h]
\renewcommand*\footnoterule{}
\centering
\begin{minipage}[t]{.5\textwidth}
\begin{algorithm}[H]
   \caption{\fastmf-LRR}
   \label{alg:dfc_lrr}
\begin{algorithmic}
	 \STATE {\bfseries Input:} $\M$, $t$
	 \STATE $\{\C_i\}_{1\le i \le t}$ = \textsc{SampleCols}($\M$, $t)$ \\
   \textbf{do in parallel}
   \STATE\hspace{5mm} $\hat{\Z}_1$ = \textsc{LRR}$(\C_1, \M)$ \\
   \STATE\hspace{20mm} \vdots \\
   \vspace{1mm}
   \STATE\hspace{5mm} $\hat{\Z}_t$ = \textsc{LRR}$(\C_t, \M)$\\
   \textbf{end do}
	 \STATE $\Zprojmf$ = \textsc{ColProj}($[\hat{\Z}_1, \ldots, \hat{\Z}_t],\hat{\Z}_1$)
\end{algorithmic}
\end{algorithm}
\end{minipage}
\end{figure}

\section{Theoretical Analysis}
\label{sec:theory}
Despite the significant reduction in computational complexity, \fastmf-LRR provably maintains the strong theoretical guarantees of the LRR algorithm.
To make this statement precise, we first review the technical conditions for accurate row space recovery required by LRR.

\subsection{Conditions for LRR Correctness}
The LRR analysis of Liu et al.~\cite{LiuXuYa11} relies on two key quantities, the rank of the clean data matrix $\mL_0$
and the \emph{coherence}~\cite{Recht09} of the singular vectors $\V_{L_0}$.
We combine these properties into a single definition:
\begin{definition}[$(\mu,r)$-Coherence] \label{eq:coherence_0}
A matrix $\mL\in\R^{m\times n}$ is \emph{$(\mu,r)$-coherent} if $\rank{\mL} = r$ and 
\begin{equation*}
\frac{n}{r} \norm{\V_{L}^\top}_{2,\infty}^2  \leq \mu,
\end{equation*}
where $\norm{\cdot}_{2,\infty}$ is the maximum column $\ell_2$ norm.\footnote{Although~\cite{LiuXuYa11} uses the notion of column
coherence to analyze LRR, we work with the closely related notion of
$(\mu,r)$-coherence for ease of notation in our proofs. Moreover, we note that if a rank-$r$ matrix
$\mL\in \reals^{m \times n}$ is supported on $(1-\gamma) n$ columns then the
column coherence of $\V_L$ is $\mu$ if and only if $\V_L$ is $(\mu/(1-\gamma), r)$-coherent.}
\end{definition}
\noindent Intuitively, when the coherence $\mu$ is small, information is well-distributed across the rows of a matrix, and
the row space is easier to recover from outlier corruption.
Using these properties, Liu et al.~\cite{LiuXuYa11} established the following recovery guarantee for LRR.

\begin{theorem}[\cite{LiuXuYa11}] \label{thm:lrr}
Suppose that $\M = \mL_0 + \S_0 \in\reals^{m\times n}$ where $\S_0$ is supported on $\gamma n$ columns, $\mL_0$ is
$(\frac{\mu}{1-\gamma},r)$-coherent, and $\mL_0$ and $\S_0$ have
independent column support with $\range{\mL_0} \cap \range{\S_0} = \{\mbf0\}$.
Let $\hat\Z$ be a solution returned by LRR.
Then there exists a constant $\gamma^*$ (depending on $\mu$ and $r$) for which
the column space of $\hat\Z$ exactly equals the row space of $\mL_0$ whenever
$\lambda=3/(7\norm{\M}\sqrt{\gamma^*l})$ and $\gamma \leq \gamma^*$.
\end{theorem}

In other words, LRR can exactly recover the row space of $\mL_0$ even when a constant fraction $\gamma^*$ of the columns has been corrupted by outliers.  As the rank $r$ and coherence $\mu$ shrink, $\gamma^*$ grows allowing greater outlier tolerance.

\begin{figure*}
\begin{center}
\begin{tabular} {@{}c@{}c@{}c@{}}
\ipsfig{.31}{figure=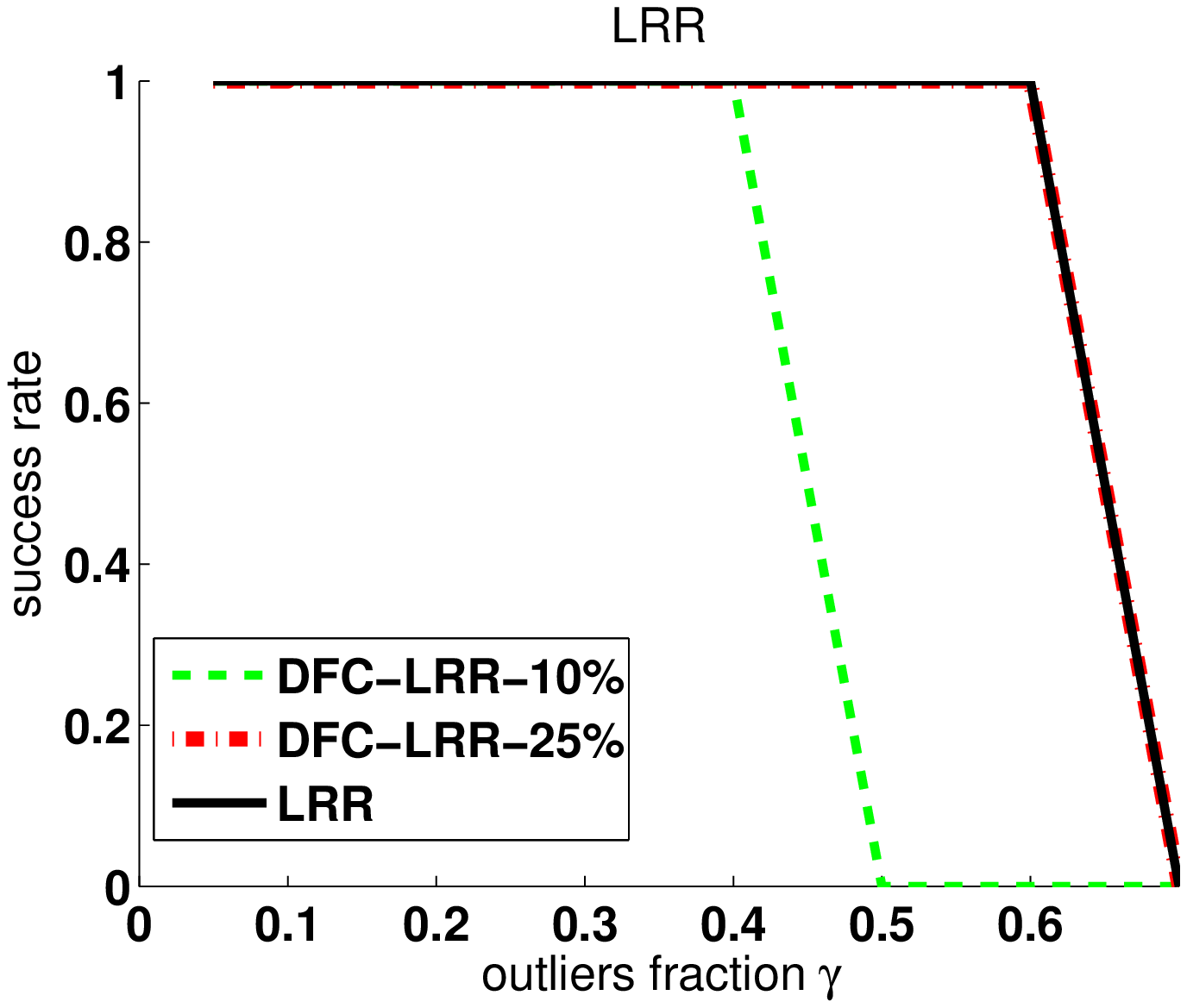} \qquad\qquad &
\ipsfig{.31}{figure=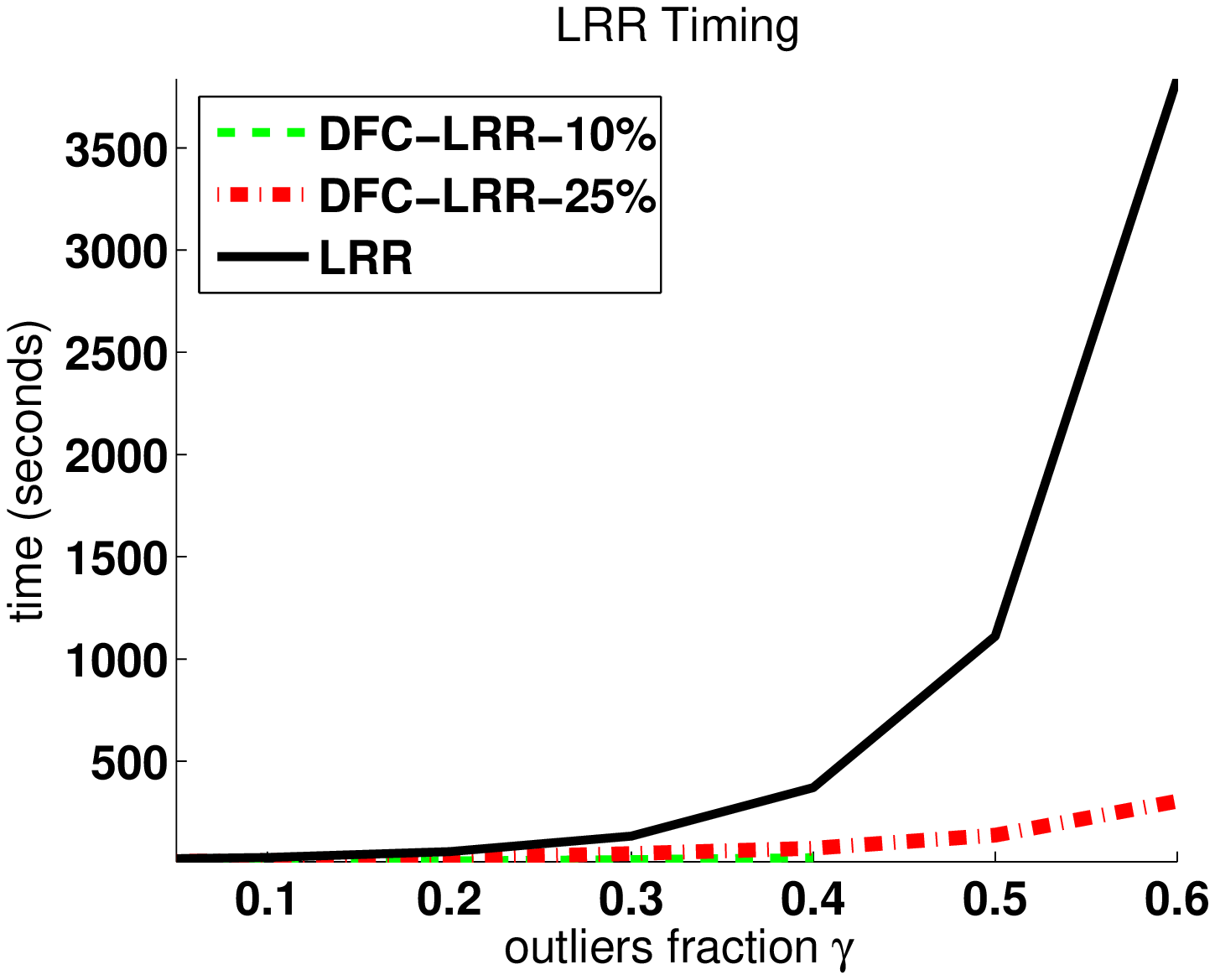} \qquad\qquad &
\ipsfig{.31}{figure=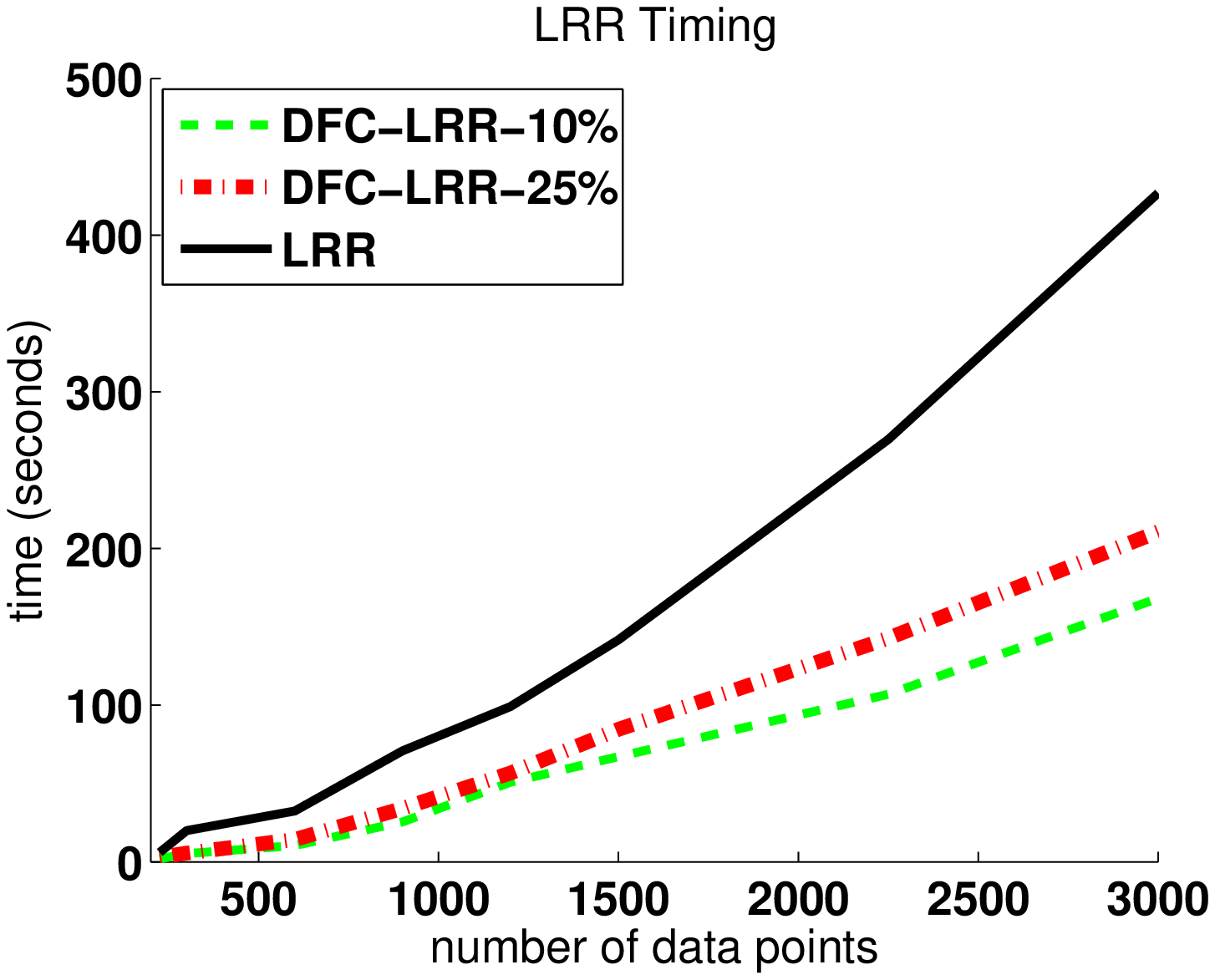} \\
(a) \qquad \qquad& (b)\qquad\qquad & (c) \\
\end{tabular}
\end{center}
\caption{Results on synthetic data (reported results are averages over $10$
trials). (a) Phase transition of LRR and \fastmf-LRR. (b,c) Timing results of LRR
and \fastmf-LRR as functions of $\gamma$ and $n$ respectively.}
\label{fig:synth_low_rank}
\end{figure*}

\subsection{High Probability Subspace Segmentation} \label{sec:lrr-theory}
Our main theoretical result shows that, with high probability and under the same conditions that guarantee the accuracy of LRR, \fastmf-LRR also exactly recovers the
row space of $\mL_0$. 
Recall that in our independent subspace setting accurate row space recovery is tantamount to correct segmentation of the columns of $\mL_0$.
The proof of our result, which generalizes the LRR analysis of~\cite{LiuXuYa11} to a broader class of optimization problems and adapts
the \fastmf analysis of~\cite{MackeyTaJo11}, can be found in the appendix.
\begin{theorem} \label{thm:fast-lrr}
Fix any failure probability $\delta > 0$.  
Under the conditions of Thm,~\ref{thm:lrr}, let $\Zprojmf$ be a solution returned by \fastmf-LRR. 
Then there exists a constant $\gamma^*$ (depending on $\mu$ and $r$) for which
the column space of $\Zprojmf$ exactly equals the row space of $\mL_0$
whenever
$\lambda=3/(7\norm{\M}\sqrt{\gamma^*l})$ for each \fastmf-LRR subproblem, $\gamma \leq
\gamma^*$, and $t=n/l$ for $$ l \geq cr\mu\log(4n/\delta)/(\gamma^*-\gamma)^2$$
and $c$ a fixed constant larger than 1.
\end{theorem}
Thm.~\ref{thm:fast-lrr} establishes that, like LRR, \fastmf-LRR can tolerate a constant fraction of its data points being corrupted and still recover the correct subspace segmentation of the clean data points with high probability.
When the number of datapoints $n$ is large, solving LRR directly may be prohibitive, but DFC-LRR need only solve a collection of
small, tractable subproblems.  Indeed, Thm.~\ref{thm:fast-lrr} guarantees high probability recovery for DFC-LRR even when the subproblem size $l$ is logarithmic in $n$. 
The corresponding reduction in computational complexity
allows \fastmf-LRR to scale to large problems with little sacrifice in accuracy.

Notably, this column sampling complexity is better than that established by~\cite{MackeyTaJo11} in the matrix
factorization setting: 
 we require O$(r \log n)$ columns
sampled, while~\cite{MackeyTaJo11} requires in the worst case $\Omega(n)$
columns for matrix completion and $\Omega((r \log n)^2)$ for robust matrix
factorization.


\section{Experiments} \label{sec:experiments}
We now explore the empirical performance of \fastmf-LRR on a variety of
simulated and real-world datasets, first for the traditional task of robust subspace segmentation
and next for the more complex task of graph-based semi-supervised learning.  
Our experiments are designed to show the effectiveness of \fastmf-LRR both when
the theory of Section~\ref{sec:theory} holds and when it is violated.  Our
synthetic datasets satisfy the theoretical assumptions of low rank, 
incoherence, and a small fraction of corrupted columns, while our 
real-world datasets violate these criteria. 

For all of our experiments
we use the inexact Augmented Lagrange Multiplier (ALM) algorithm of
\cite{LiuLiYu10} as our base LRR algorithm.  For the subspace segmentation
experiments, we set the regularization parameter to the values suggested in
previous works \cite{LiuXuYa11,LiuLiYu10}, while in our semi-supervised
learning experiments we set it to $1/\sqrt{\max{(m,n)}}$ as suggested in prior
work.\footnote{\url{http://perception.csl.illinois.edu/matrix-rank}} In all
experiments we report parallel running times for \fastmf-LRR, i.e., the time of
the longest running subproblem plus the time required to combine submatrix estimates
via column projection.  All experiments were implemented in Matlab. The
simulation studies were run on an x$86$-$64$ architecture using a single $2.60$
Ghz core and $30$GB of main memory, while the real data experiments were
performed on an x$86$-$64$ architecture equipped with a 2.67GHz 12-core CPU and
64GB of main memory.

\subsection{Subspace Segmentation: LRR vs. DFC-LRR}
\label{ssec:sub_seg}
We first aim to verify that DFC-LRR produces accuracy comparable to LRR in significantly less time, both in synthetic and real-world settings.
We focus on the standard robust subspace segmentation task of identifying the subspace associated with each input datapoint.

\subsubsection{Simulations} \label{ssec:simulations}
To construct our synthetic robust subspace segmentation datasets, we first generate $n_s$ datapoints from each of $k$ independent $r$-dimensional subspaces of $\reals^{m}$,
in a manner similar to~\cite{LiuXuYa11}.
For each subspace $i$, we independently select a basis
$\U_i$ uniformly from all matrices in $\reals^{m \times r}$ with orthonormal columns
and a matrix $\T_i\in \reals^{r \times n_s}$ of independent entries each distributed uniformly in $[0,1]$.  We form the matrix $\X_i \in \reals^{m \times n_s}$ of samples 
from subspace $i$ via $\X_i = \U_i\T_i$  and let
$\X_0 \in \reals^{m \times kn_s} = [\X_1 \,\, \ldots \,\, \X_k]$. For a given
outlier fraction $\gamma$ we next generate an additional $n_o =
\frac{\gamma}{1-\gamma} kn_s$ independent outlier samples, denoted by $\S \in \reals^{m
\times n_o}$.  Each outlier sample has independent $\Gsn(0,\sigma^2)$ entries,
where $\sigma$ is the average absolute value of the entries of the $kn_s$
original samples. We create the input matrix $\M \in \reals^{m \times n
}$, where $n = kn_s + n_o$, as a random permutation of the columns of $[\X_0 \quad \S]$.

\begin{figure}
\centering
  \includegraphics[width=\linewidth]{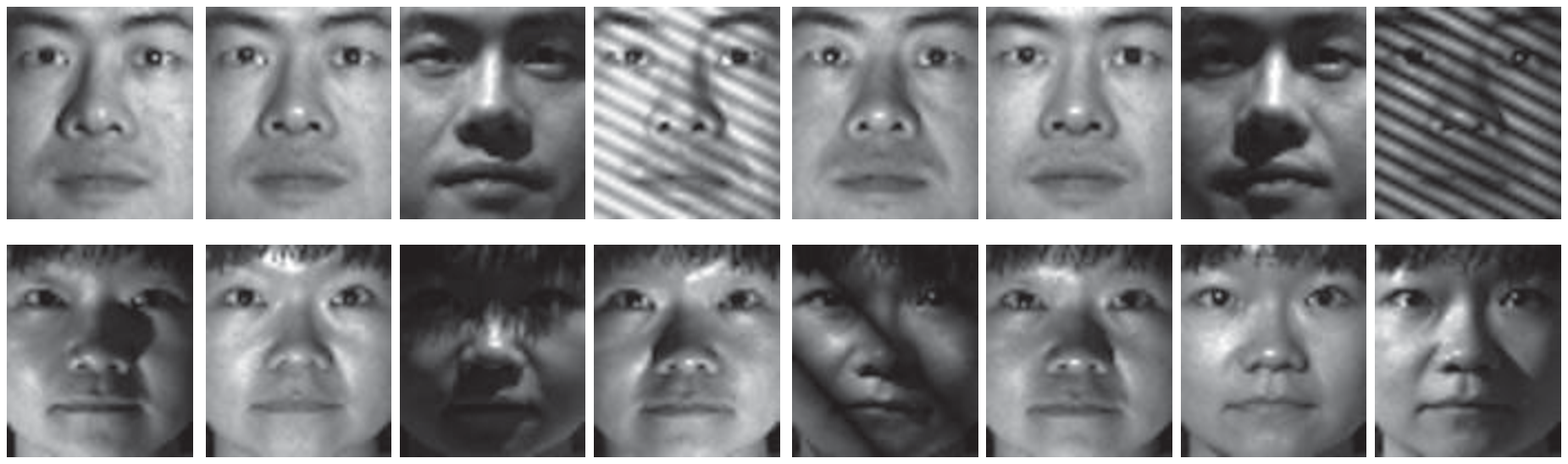}
  \caption{Exemplar face images from Extended Yale Database B. Each row
shows randomly selected images for a human subject.}
  \label{fig:yaleb-example}
\end{figure}

\begin{figure*}
\begin{center}
\begin{tabular} {@{}c@{}c@{}}
\ipsfig{.41}{figure=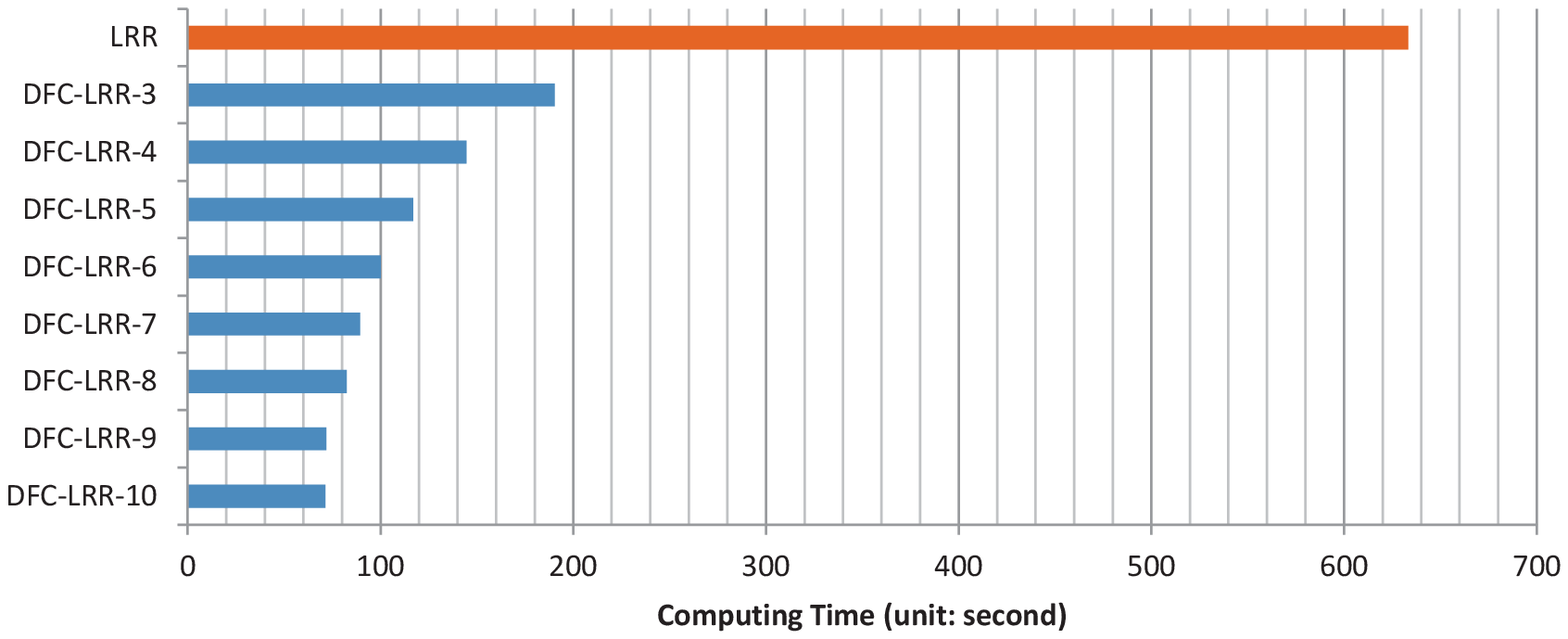}\qquad\qquad &
\ipsfig{.41}{figure=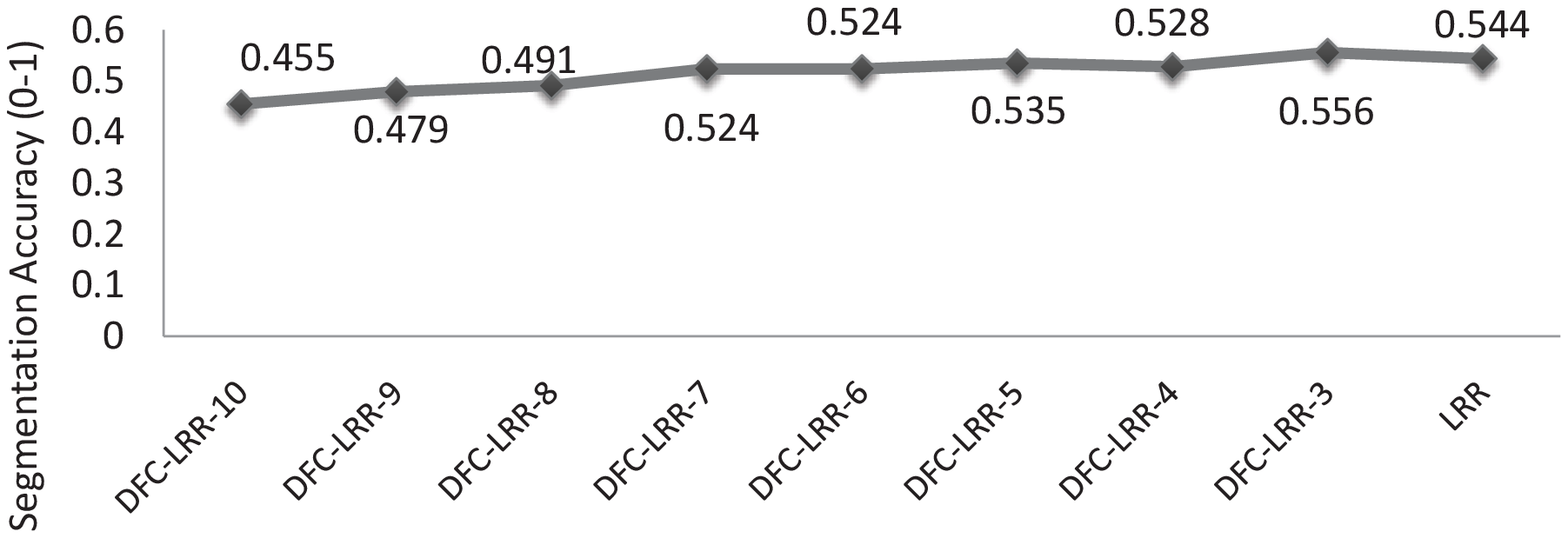} \\
(a) \qquad\qquad& (b) \\
\end{tabular}
\end{center}
\caption{Trade-off between computation and segmentation accuracy on face
recognition experiments. All results are obtained by averaging across 100
independent runs. (a) Run time of LRR and \fastmf-LRR with varying number of
subproblems. (b) Segmentation accuracy for these same experiments. }
\label{fig:yaleb}
\end{figure*}

In our first experiments we fix $k=3$, $m=1500$, $r=5$, and $n_s = 200$, set the
regularizer to $\lambda = 0.2$, and vary the fraction of outliers.  We measure with what frequency
LRR and \fastmf-LRR are able to recover of the row space of $\X_0$ and identify the outlier columns in $\S$,
using the same criterion as defined
in~\cite{LiuXuYa11}.\footnote{Success is determined by whether the oracle
constraints of \eqref{eq:oracle-constraints} in the Appendix are satisfied
within a tolerance of $10^{-4}$.} Figure~\ref{fig:synth_low_rank}(a) 
shows
average performance over $10$ trials. We see that
\fastmf-LRR performs quite well, as the gaps in the phase transitions between
LRR and \fastmf-LRR are small when sampling $10\%$ of the columns (i.e.,
$t=10$) and are virtually non-existent when sampling $25\%$ of the columns
(i.e., $t=4$).

Figure~\ref{fig:synth_low_rank}(b) shows
corresponding timing results for the accuracy results presented in
Figure~\ref{fig:synth_low_rank}(a). These timing results show substantial speedups in
\fastmf-LRR relative to LRR with a modest tradeoff in accuracy as denoted in
Figure~\ref{fig:synth_low_rank}(a). Note that we only report timing results for
values of $\gamma$ for which \fastmf-LRR was successful in all $10$ trials,
i.e., for which the success rate equaled $1.0$ in Figure~\ref{fig:synth_low_rank}(a).
Moreover, Figure~\ref{fig:synth_low_rank}(c) shows timing results using the same
parameter values, except with a fixed fraction of outliers ($\gamma = 0.1$) and
a variable number of samples in each subspace, i.e., $n_s$ ranges from $75$ to
$1000$.  These timing results also show speedups with minimal loss of accuracy,
as in all of these timing experiments, LRR and \fastmf-LRR were successful in
all trials using the same criterion defined in~\cite{LiuXuYa11} and used in our
phase transition experiments of Figure~\ref{fig:synth_low_rank}(a).

\subsubsection{Face Clustering}
\label{sec:scene}

We next demonstrate the comparable quality and increased performance of
\fastmf-LRR relative to LRR on real data, namely,
a subset of Extended
Yale Database B,\footnote{\url{http://vision.ucsd.edu/~leekc/ExtYaleDatabase}}
a standard face benchmarking dataset. Following the experimental
setup in~\cite{LiuLiYu10}, $640$ frontal face images of $10$ human subjects are
chosen, each of which is resized to be $48 \times 42$ pixels and forms a
2016-dimensional feature vector. As noted in previous work~\cite{CandesLiMaWr09}, a
low-dimensional subspace can be effectively used to model face images from one
person, and hence face clustering is a natural application of subspace
segmentation.  Moreover, as illustrated in Figure~\ref{fig:yaleb-example}, a
significant portion of the faces in this dataset are ``corrupted" by shadows,
and hence this collection of images is an ideal benchmark for \emph{robust}
subspace segmentation.

As in~\cite{LiuLiYu10}, we use the feature vector representation of these
images to create a $2016 \times 640$ dictionary matrix, $\M$, and run both LRR
and \fastmf-LRR with the parameter $\lambda$ set to $0.15$. Next, we use the
resulting low-rank coefficient matrix $\hat{\Z}$ to compute an affinity matrix
$\U_{\hat{Z}}\U_{\hat{Z}}^\top$, where $\U_{\hat{Z}}$ contains the top left singular vectors of $\hat{\Z}$.
The affinity matrix is used to cluster the data into $k=10$ clusters
(corresponding to the $10$ human subjects) via spectral embedding (to obtain a
$10$D feature representation) followed by $k$-means.
Following~\cite{LiuLiYu10}, the comparison of different clustering methods
relies on \emph{segmentation accuracy}. Each of the $10$ clusters is assigned a
label based on majority vote of the ground truth labels of the points assigned
to the cluster.
We evaluate clustering performance of both LRR and \fastmf-LRR by computing
segmentation accuracy as in~\cite{LiuLiYu10}, i.e., each cluster
is assigned a label based on majority vote of the ground truth labels of the
points assigned to the cluster.  The segmentation accuracy is then computed by
averaging the percentage of correctly classified data over all classes.

Figures~\ref{fig:yaleb}(a) and~\ref{fig:yaleb}(b) show the computation time and
the segmentation accuracy, respectively, for LRR and for
\fastmf-LRR with varying numbers of subproblems (i.e., values of $t$).  On this
relatively-small data set ($n=640$ faces), LRR requires over $10$ minutes to
converge.
\fastmf-LRR demonstrates a roughly linear computational speedup as a
function of $t$, comparable accuracies to LRR for smaller values of $t$ and a
quite gradual decrease in accuracy for larger $t$.

\subsection{Graph-based Semi-Supervised Learning}
\label{ssec:semi_super}

Graph representations, in which samples are vertices and weighted edges express affinity
relationships between samples, are crucial in various computer vision tasks.
Classical graph construction methods separately calculate the outgoing edges
for each sample.  This local strategy makes the graph vulnerable to
contaminated data or outliers.
Recent work in computer vision has illustrated the utility of global graph
construction strategies using graph Laplacian~\cite{gao10}
or matrix low-rank~\cite{zhuang12} based regularizers. L1 regularization has also
been effectively used to encourage sparse graph construction~\cite{cheng10,he11}.
Building upon the success of global construction methods and noting the
connection between subspace segmentation and graph construction as described in
Section~\ref{ssec:subspace_segmentation}, we present a novel application of the
low-rank representation methodology, relying on our \fastmf-LRR algorithm to
scalably yield a \emph{sparse, low-rank graph} (SLR-graph).  We
present a variety of results on large-scale semi-supervised learning visual
classification tasks and provide a detailed comparison with leading baseline
algorithms.

\subsubsection{Benchmarking Data}

We adopt the following three large-scale benchmarks:

\vspace{0.05in}
\noindent \textbf{Columbia Consumer Video
(CCV) Content Detection}\footnote{\url{http://www.ee.columbia.edu/ln/dvmm/CCV/}}: Compiled to
stimulate research on recognizing highly-diverse visual content in
unconstrained videos, this dataset consists of $9317$ YouTube videos over
$20$ semantic categories (e.g., baseball, beach, music performance). Three
popular audio/visual features (5000-D SIFT, 5000-D STIP, and 4000-D MFCC) are extracted.

\vspace{0.03in}
\noindent \textbf{MED12 Multimedia Event Detection}:
The MED12 video corpus consists of $\mathtt{\sim}150$K multimedia videos, with
an average duration of $2$ minutes, and is used for detecting 20 specific
semantic events. For each event, $130$ to $367$ videos are provided as positive
examples, and the remainder of the videos are ``null" videos that do not
correspond to any event. In this work, we keep all positive examples and
sample $10$K null videos, resulting in a dataset of $13,876$ videos. We extract
six features from each video, first at sampled frames and then accumulated to
obtain video-level representations. The features are either visual (1000-D
sparse-SIFT, 1000-D dense-SIFT, 1500-D color-SIFT, 5000-D STIP), audio (2000-D
MFCC), or semantic features (2659-D CLASSEME~\cite{torr10}).

\vspace{0.05in}
\noindent \textbf{NUS-WIDE-Lite Image Tagging}: NUS-WIDE is among the largest available image
tagging benchmarks, consisting of over $269$K crawled images from Flickr that
are associated with over $5$K user-provided tags. Ground-truth images are
manually provided for $81$ selected concept tags. We generate a lite version by
sampling $20$K images. For each image, 128-D wavelet texture, 225-D block-wise
LAB-based color moments and 500-D bag of visual words are extracted, normalized
and finally concatenated to form a single feature representation for the image.

\begin{figure*}[ht!]
\begin{center}
\begin{tabular} {@{}c@{}c@{}}
\ipsfig{.42}{figure=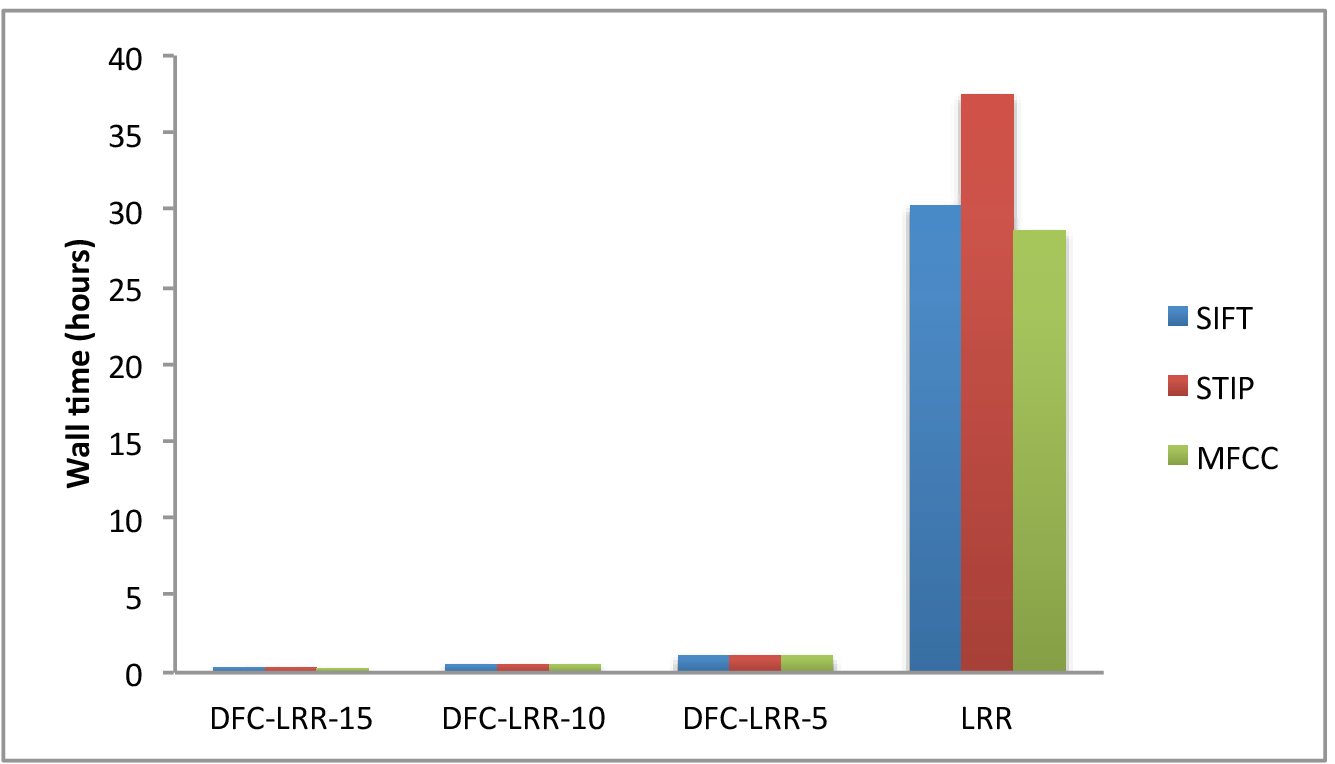}\qquad\qquad\qquad\qquad &
\ipsfig{.42}{figure=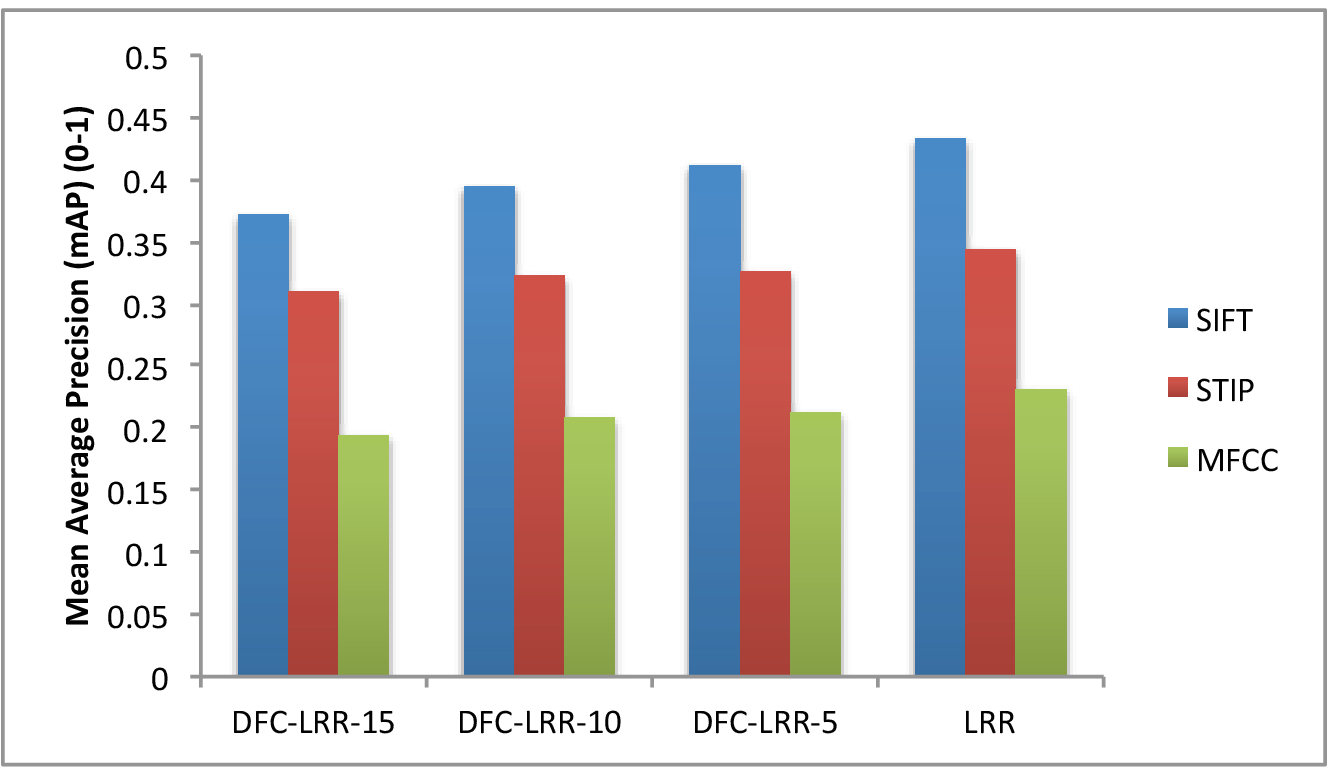} \\
(a) \qquad\qquad\qquad\qquad & (b) \\
\end{tabular}
\end{center}
\caption{ Trade-off between computation and accuracy for the SLR-graph on the CCV
dataset. (a) Wall time of LRR and
\fastmf-LRR with varying numbers of subproblems. (b) mAP scores for these same
experiments.  }
\label{fig:ccv}
\end{figure*}

\subsubsection{Graph Construction Algorithms}
\label{ssec:graph_const}
The three graph construction schemes we evaluate are described below. Note that
we exclude other baselines (e.g., NNLRS~\cite{zhuang12}, LLE
graph~\cite{wangj09}, L1-graph~\cite{cheng10}) due to either scalability
concerns
or because prior work has already demonstrated inferior performance relative to
the SPG algorithm defined below~\cite{zhuang12}.

\vspace{0.03in}
\noindent \textbf{$k$NN-graph}: We construct a nearest neighbor graph by
connecting (via undirected edges) each vertex to its $k$ nearest
neighbors in terms of $l_2$ distance in the specified feature space.
Exponential weights are associated with edges, i.e.,
$w_{ij} = \exp{- d_{ij}^2 / \sigma^2}$, where
$d_{ij}$ is the distance between $x_i$ and $x_j$ and $\sigma$ is an empirically-tuned
parameter~\cite{wang06}.

\vspace{0.03in}
\noindent \textbf{SPG}: Cheng et al.~\cite{cheng10} proposed a noise-resistant
L1-graph which encourages sparse vertex connectedness, motivated by the work of
sparse representation~\cite{wright09}. Subsequent work, entitled \emph{sparse
probability graph} (SPG)~\cite{he11} enforced positive graph weights.
Following the approach of~\cite{zhuang12}, we implemented a variant of SPG by solving
the following optimization problem for each sample:
\begin{eqnarray}
\min_{\mathbf{w}_x} \ \  \| \mathbf{x} - \mathbf{D}_x \mathbf{w}_x\|_2^2 + \alpha \|\mathbf{w}_x\|_1, \ \ s.t. \ \ \mathbf{w}_x \ge 0,
\label{eqn:spg}
\end{eqnarray}
where $\mathbf{x}$ is a feature representation of a sample and $\mathbf{D}_x$
is the basis matrix for $\x$ constructed from its $n_k$ nearest neighbors.
We use an open-source tool\footnote{\url{http://sparselab.stanford.edu}} to
solve this non-negative Lasso problem.

\vspace{0.03in}
\noindent \textbf{SLR-graph}:
Our novel graph construction method contains two-steps: first LRR or
\fastmf-LRR is performed on the entire data set to recover the intrinsic
low-rank clustering structure.  We then treat the resulting low-rank coefficient
matrix $\Z$ as an affinity matrix, and
for sample $x_i$, the $n_k$ samples with largest affinities to $x_i$ are
selected to form a basis matrix and used to solve the SPG optimization
described by Problem~(\ref{eqn:spg}). The resulting non-negative coefficients
(typically sparse owing to the $\ell_1$ regularization term on $\mathbf{w}_x$
in ~(\ref{eqn:spg})) are used to define the graph.

\begin{table}[t]
\caption{Mean average precision (mAP) (0-1) scores for various graph
construction methods. \fastmf-LRR-10
is performed for SLR-Graph.
The best mAP score for each feature is highlighted in bold.
\ignore{\notate{I believe we're reporting results with \fastmf-LRR-10; results with $t=5$
would look much better and would still give us an order of magnitude speedup.  Regardless
of what we decide, we need to clearly specify what value of $t$ we're using in these
reported results}
}
}
\label{table:graph}

\begin{small}
\begin{sc}

\begin{subtable}{\linewidth}
\caption{CCV}
\centering
\begin{tabular}{lccc}
\hline
 & $k$NN-Graph & SPG & SLR-Graph\\
\hline
SIFT & .2631 & .3863 & \textbf{.3946}\\
STIP   & .2011 & .3036 & \textbf{.3227}\\
MFCC   & .1420 & \textbf{.2129} & .2085\\
\hline
\end{tabular}
\end{subtable}

\vspace{0.1cm}

\begin{subtable}{\linewidth}
\caption{MED12}
\centering
\begin{tabular}{lccc}
\hline
 & $k$NN-Graph & SPG & SLR-Graph\\
\hline
Color-SIFT & .0742 & .1202 & \textbf{.1432}\\
Dense-SIFT   & .0928 & .1350 & \textbf{.1525}\\
Sparse-SIFT   & .0780 & .1258 & \textbf{.1464}\\
MFCC   & .0962 & \textbf{.1371} & \textbf{.1371}\\
CLASSEME   & .1302 & .1872 & \textbf{.2120}\\
STIP & .0620 & \textbf{.0835} & .0803 \\
\hline
\end{tabular}
\end{subtable}

\vspace{0.1cm}

\begin{subtable}{\linewidth}
\caption{NUS-WIDE-Lite}
\centering
\begin{tabular}{ccc}
\hline
$k$NN-Graph & SPG & SLR-Graph\\
\hline
 .1080 & .1003 & \textbf{.1179}\\
\hline
\end{tabular}
\end{subtable}

\end{sc}
\end{small}
\end{table}

\subsubsection{Experimental Design}

For each benchmarking dataset, we first construct graphs by treating sample
images/videos as vertices and using the three algorithms outlined in
Section~\ref{ssec:graph_const} to create (sparse) weighted edges between
vertices. For fair comparison, we use the same parameter settings, namely
$\alpha = 0.05$ and $n_k = 500$ for both SPG and SLR-graph.  Moreover, we set
$k=40$ for $k$NN-graph after tuning over the range $k=10$ through $k=60$.

We then use a given graph structure to perform semi-supervised label
propagation using an efficient label propagation algorithm~\cite{wang06} that
enjoys a closed-form solution and often achieves the state-of-the-art
performance. We perform a separate label propagation for each category in our
benchmark, i.e., we run a series of $20$ binary classification label propagation
experiments for CCV/MED12 and $81$ experiments for NUS-WIDE-Lite.  For each
category, we randomly select half of the samples as training points (and use
their ground truth labels for label propagation) and use the remaining half 
as a test set. We repeat this process $20$ times for
each category with different random splits.  Finally, we compute Mean Average
Precision (mAP) based on the results on the test sets across all runs of label
propagation.

\subsubsection{Experimental Results}

We first performed experiments using the CCV benchmark, the smallest of our
datasets, to explore the tradeoff between computation and accuracy when
using \fastmf-LRR as part of our proposed SLR-graph.  Figure~\ref{fig:ccv}(a)
presents the time required to run SLR-graph with LRR versus \fastmf-LRR with
three different numbers of subproblems ($t = 5, 10, 15$), while
Figure~\ref{fig:ccv}(b) presents the corresponding accuracy results. The
figures show that \fastmf-LRR performs comparably to LRR for smaller values of
$t$, and performance gradually degrades for larger $t$.  Moreover, \fastmf-LRR
is up to two orders of magnitude faster and achieves superlinear speedups relative to
LRR.\footnote{We restricted the maximum number of internal LRR iterations to $500$ to
ensure that LRR ran to completion in less than two days.} Given
the scalability issues of LRR on this modest-sized dataset, along with the
comparable accuracy of \fastmf-LRR, we ran SLR-graph exclusively with
\fastmf-LRR ($t=10$) for our two larger datasets. 

Table~\ref{table:graph} summarizes the results of our semi-supervised learning
experiments using the three graph construction techniques defined in
Section~\ref{ssec:graph_const}.  The results show that our proposed SLR-graph
approach leads to significant performance gains in terms of mAP across all
benchmarking datasets for the vast majority of features. These results
demonstrate the benefit of enforcing both low-rankedness and sparsity during
graph construction.  Moreover, conventional low-rank oriented algorithms, e.g.,
\cite{zhuang12, LinGaWrWuChMa09} would be computationally infeasible on our
benchmarking datasets, thus highlighting the utility of employing \fastmf's
divide-and-conquer approach to generate a scalable algorithm.

\section{Conclusion}
\label{sec:conclusion}

Our primary goal in this work was to introduce a provably accurate algorithm suitable for large-scale
low-rank subspace segmentation.
While some contemporaneous work~\cite{AdlerElHe13} also aims at scalable subspace segmentation, 
this method offers no guarantee of correctness.
In contrast, \fastmf-LRR provably preserves the theoretical recovery
guarantees of the LRR program.
Moreover, our divide-and-conquer approach achieves empirical accuracy
comparable to state-of-the-art methods
while obtaining linear to superlinear computational gains, both on standard subspace
segmentation tasks and on novel applications to semi-supervised learning. 
\fastmf-LRR also lays the groundwork for scaling up 
LRR derivatives known to offer improved performance, e.g., LatLRR 
in the setting of standard subspace segmentation and NNLRS in the graph-based 
semi-supervised learning setting.
The same techniques may prove useful in developing scalable approximations
to other convex formulations for subspace segmentation, e.g., \cite{LiuYa12}.  

{\small
\bibliographystyle{ieee}
\bibliography{refs}
}
\clearpage
\newpage
\appendix

\section{Proof of Theorem~\ref{thm:fast-lrr}}

Our proof of Thm.~\ref{thm:fast-lrr} rests upon three key results:
a new deterministic recovery guarantee for LRR-type problems that generalizes the guarantee 
of \cite{LiuXuYa11},
a probabilistic estimation guarantee for column projection established
in \cite{MackeyTaJo11},
and a probabilistic guarantee of \cite{MackeyTaJo11} showing that a uniformly chosen submatrix
of a $(\mu,r)$-coherent matrix is nearly $(\mu,r)$-coherent.
These results are presented in Secs.~\ref{sec:lrr-analysis}, \ref{sec:rp-analysis}, and
\ref{sec:subcoh} respectively.
The proof of Thm.~\ref{thm:fast-lrr} follows in Sec.~\ref{sec:fast-lrr-proof}.

In what follows, the unadorned norm $\norm{\cdot}$ represents the spectral norm of a matrix.
We will also make use of a
technical condition, introduced by Liu et al.~\cite{LiuXuYa11} to ensure that a
corrupted data matrix is well-behaved when used as a dictionary:
\begin{definition}[Relatively Well-Definedness]
A matrix $\M = \mL_0 + \S_0$ is $\beta$-RWD if
$$\norm{\mSigma_{M}^{-1}\V_{M}^T\V_{L_0}} \leq \frac{1}{\beta \norm{\M}}.$$
\end{definition}
\noindent A larger value of $\beta$ corresponds to improved recovery properties.

\subsection{Analysis of Low-Rank Representation} \label{sec:lrr-analysis}
\newcommand{\mO}{\mathbf{O}}
Thm.~1 of \cite{LiuXuYa11} analyzes LRR recovery under the constraint $\mO = \D\Z + \S$
when the observation matrix $\mO$ and the dictionary $\D$ are both equal to the input matrix $\M$.
Our next theorem provides a comparable analysis when the observation matrix is a column submatrix 
of the dictionary.

\begin{theorem} \label{thm:sub-lrr}
Suppose that $\M = \mL_0 + \S_0 \in \reals^{m\times n}$ is $\beta$-RWD with rank $r$ 
and that $\mL_0$ and $\S_0$ have independent column support with $\range{\mL_0} \cap \range{\S_0} = \{\mbf0\}$.
Let $\S_{0,C}\in\reals^{m\times l}$ be a column submatrix of $\S_0$ supported on $\gamma l$ columns,
and suppose that $\C$, the corresponding column submatrix of $\M$, is $(\frac{\mu}{1-\gamma},r)$-coherent.
Define
$$\gamma^* \defeq \frac{324 \beta^2}{324\beta^2 + 49(11+4\beta)^2\mu r},$$
and let $(\hat\Z,\hat\S)$ be a solution to the problem
\begin{align} \label{eq:sub-lrr}
&\min_{\Z,\S} \quad \norm{\Z}_*+\lambda\norm{\S}_{2,1}\quad
\text{subject\, to}\quad \C = \M\Z + \S
\end{align}
with $\lambda=3/(7\norm{\M}\sqrt{\gamma^*l})$.
If $\gamma \leq \gamma^*$, then
the column space of $\hat\Z$ equals the row space of $\mL_0$.
\end{theorem}
The proof of Thm.~\ref{thm:sub-lrr} can be found in Sec.~\ref{sec:sub-lrr-proof}.

\subsection{Analysis of Column Projection} \label{sec:rp-analysis}
The following lemma, due to \cite{MackeyTaJo11}, shows that, 
with high probability,
column projection exactly recovers a $(\mu,r)$-coherent matrix
by sampling a number of columns proportional to $\mu r\log n$.

\begin{corollary}[Column Projection under Incoherence {\cite[Cor.~6]{MackeyTaJo11}}]
\label{cor:proj-main}
Let $\mL\in\reals^{m\times n}$ be $(\mu, r)$-coherent, and let $\mL_C\in\reals^{m\times
l}$ be a matrix of $l$ columns of $\mL$ sampled uniformly without replacement.
If $l \geq cr\mu\log(n)\log(1/\delta),$ where $c$ is a fixed positive
constant, then, $$\mL = \Lproj \defeq \U_{L_C}\U_{L_C}^\top\mL$$ exactly with probability at least $1-\delta$.
\end{corollary}

\subsection{Conservation of Incoherence} \label{sec:subcoh}
The following lemma of \cite{MackeyTaJo11} shows that, with high probability,
$\mL_{0,i}$ captures the full rank of $\mL_0$ and has coherence not much larger than $\mu$.

\begin{lemma}[Conservation of Incoherence {\cite[Lem.~7]{MackeyTaJo11}}] \label{lem:sub-coh}
Let $\mL\in\reals^{m\times n}$ be $(\mu, r)$-coherent, and let $\mL_C\in\reals^{m\times
l}$ be a matrix of $l$ columns of $\mL$ sampled uniformly without replacement.
If $l \geq cr\mu\log(n)\log(1/\delta)/\epsilon^2,$ where $c$ is a fixed
constant larger than 1, then $\mL_C$ is $(\frac{\mu}{1-\epsilon/2}, r)$-coherent
with probability at least $1-\delta/n$.
\end{lemma}

\subsection{Proof of DFC-LRR Guarantee} \label{sec:fast-lrr-proof}
Recall that, under Alg.~\ref{alg:dfc_lrr}, the input matrix $\M$ has been partitioned 
into column submatrices $\{\C_1,\dots,\C_t\}$.
Let $\{\C_{0,1},\dots,\C_{0,t}\}$ and $\{\S_{0,1},\dots,\S_{0,t}\}$ be the corresponding partitions
of $\mL_0$ and $\S_0$,
let $s_i \defeq \gamma_i l$ be the size of the column support of $\S_{0,i}$ for each index $i$,
and let $(\hat\Z_i, \hat\S_i)$ be a solution to the $i$th DFC-LRR subproblem.

For each index $i$, we further define $A_i$ as the event that $\C_{0,i}$ is $(4\mu/(1-\gamma_i),r)$-coherent,
$B_i$ as the event that $s_i \leq \gamma^* l$, and $G(\Z)$ as the event that the column space
of the matrix $\Z$ is equal to the row space of $\mL_0$.
Under our choice of $\gamma^*$, Thm.~\ref{thm:sub-lrr} implies that $G(\hat\Z_i)$ holds when $A_i$ and $B_i$ are both realized.
Hence, when $A_i$ and $B_i$ hold for all indices $i$, the column space of 
$\hat\Z = [\hat\Z_1,\dots,\hat\Z_t]$ precisely equals the row space of $\mL_0$,
and the median rank of $\{\hat\Z_1,\dots,\hat\Z_t\}$ equals $r$.

Applying Cor.~\ref{cor:proj-main} with 
$$l \geq cr\mu\log^2(4n/\delta)/(\gamma^*-\gamma)^2 \geq
cr\mu\log(n)\log(4/\delta),$$
shows that, given $A_i$ and $B_i$ for all indices $i$, $\Zprojmf$ equals $\hat\Z$ with probability at least $1-\delta/4$.
To establish $G(\Zrpmf)$ with probability at least $1-\delta$, it therefore remains to show that 
\begin{align}
\Parg{\cap_{i=1}^t (A_i \cap B_i)} & = 1 - \Parg{\cup_{i=1}^t (A_i^c \cup B_i^c)} \\ 
& \geq 1 - \sum_{i=1}^t(\Parg{A_i^c} + \Parg{B_i^c}) \\
& \geq 1 - 3\delta/4.
\end{align}

Because \fastmf-LRR partitions columns uniformly at random, 
the variable $s_i$ has a hypergeometric distribution with $\E{s_i} = \gamma l$
and therefore satisfies Hoeffding's inequality for the
hypergeometric distribution~\cite[Sec.~6]{Hoeffding63}: 
$$\Parg{s_i \geq \E{s_i} + l\tau} \leq \exp{-2lt^2}.$$
It follows that	
\begin{align*}
\Parg{B_i^c} = \Parg{s_i > \gamma^* l } &= \Parg{s_i > \E s_i + l(\gamma^* - \gamma)} \\
& \leq \exp{-2l(\gamma^* - \gamma)^2} \leq \delta/(4t) 
\end{align*}
by our assumption that 
$l \geq cr\mu\log^2(4n/\delta)/(\gamma^*-\gamma)^2 \geq \log(4t/\delta)/[2(\gamma^* - \gamma)^2]$.

By Lem.~\ref{lem:sub-coh} and our choice of

\begin{align*}
l & \geq cr\mu\log^2(4n/\delta)/(\gamma^*-\gamma)^2 \\
& \geq cr\mu\log(n)\log(4/\delta)/(1-\gamma),
\end{align*}
each submatrix $\C_{0,i}$ is $(2\mu/(1-\gamma),r)$-coherent with probability at
least $1-\delta/(4n)\geq 1-\delta/(4t)$.  A second application of Hoeffding's
inequality for the hypergeometric further implies that
\begin{align*}
\Parg{\frac{2\mu}{1 - \gamma}  > \frac{4\mu}{1 - \gamma_i} } &= \Parg{s_i < \E
s_i - l(1- \gamma) } \\ 
& \leq \exp{-2l(1 - \gamma)^2} \\
& \leq \delta/(4t),
\end{align*}
since $l \geq cr\mu\log(4n/\delta)/(\gamma^*-\gamma)^2 \geq \log(4t/\delta)/[2(1 - \gamma)^2]$.
Hence, $\Parg{A_i^c} \leq \delta/(2t)$.

Combining our results, we find
$$ \sum_{i=1}^t(\Parg{A_i^c} + \Parg{B_i^c}) \leq 3\delta/4 $$
as desired.

\section{Proof of Theorem~\ref{thm:sub-lrr}} \label{sec:sub-lrr-proof}
Let $\badcols$ be the column support of $\S_{0,C}$, and let $\badcolsc$ be its
set complement in $\{1,\dots,l\}$.
For any matrix $\S\in\reals^{a\times b}$ and index set $\mc{I}\subseteq\{1,\dots,b\}$,
we let $\proj_{\mc{I}}(\S)$ be the orthogonal projection of $\S$ onto the space of 
$a \times b$ matrices with column support $\mc{I}$, so
that $(\proj_{\mc{I}}(\S))^{(j)} = \S^{(j)},\ \text{if}\ j\in\mc{I}\quad
\text{and}\quad (\proj_{\mc{I}}(\S))^{(j)} = \mbf0\ \text{otherwise}.$
\subsection{Oracle Constraints}
Our proof of Thm.~\ref{thm:sub-lrr} will parallel Thm.~1 of \cite{LiuXuYa11}.
We begin by introducing two oracle constraints that would guarantee the 
desired outcome if satisfied.
\begin{lemma} \label{lem:oracle-constraints}
	Under the assumptions of Thm.~\ref{thm:sub-lrr}, suppose that 
   $\C = \M\Z + \S$ for some matrices $(\Z,\S)$.
	If $(\Z,\S)$ additionally satisfy the \emph{oracle constraints}
	\begin{align} \label{eq:oracle-constraints}
		\P_{L_{0}^\top}\Z = \Z \quad\text{and}\quad\badcolproj(\S) = \S
	\end{align}
	then the column space of $\Z$ equals the row space of $\mL_{0}$.
\end{lemma}
\begin{proof}
By Eq.~\ref{eq:oracle-constraints}, the row space of $\mL_{0}$ contains the column space of $\Z$, so
the two will be equal if $\rank{\mL_{0}} = \rank{\Z}$.
This equality indeed holds, since
$$\C_0 = \badcolprojc(\C) = \badcolprojc(\M\Z + \S) = \M\badcolprojc(\Z),$$
and therefore
$\rank{\mL_0} = \rank{\C_0} \leq \rank{\M\badcolprojc(\Z)} \leq \rank{\badcolprojc(\Z)} \leq \rank{\Z} \leq \rank{\mL_0}.$
\end{proof}
Thus, to prove Thm.~\ref{thm:sub-lrr}, it suffices to show that any solution to Eq.~\ref{eq:sub-lrr}
also satisfies the oracle constraints of Eq.~\ref{eq:oracle-constraints}.

\subsection{Conditions for Optimality}
To this end, we derive sufficient conditions for solving Eq.~\ref{eq:sub-lrr} and moreover show
that if any solution to Eq.~\ref{eq:sub-lrr} satisfies the oracle constraints of Eq.~\ref{eq:oracle-constraints},
then all solutions do.

We will require some additional notation.
For a matrix $\Z\in\reals^{n\times l}$ we define $T(\Z) \defeq \{\U_Z\X + \Y\V_Z^\top : \X\in\reals^{r\times l}, \Y\in\reals^{n\times r}\}$, $\proj_{T(\Z)}$ as the orthogonal projection onto the set $T(\Z)$, and 
$\proj_{T(\Z)^\bot}$ as the orthogonal projection onto the orthogonal complement of $T(\Z)$.
For a matrix $\S$ with column support $\mc{I}$, we define the column normalized version, $\mc{B}(\S)$, which satisfies
$$\proj_{\mc{I}^c}(\mc{B}(\S)) = \mbf0\quad\text{and}\quad\mc{B}(\S)^{(j)} \defeq \S^{(j)}/\norm{\S^{(j)}}\quad\forall j\in\mc{I}.$$

\begin{theorem} \label{thm:optimality}
	Under the assumptions of Thm.~\ref{thm:sub-lrr}, suppose that 
	$\C = \M\Z + \S$ for some matrices $(\Z,\S)$.
	If there exists a matrix $\Q$ satisfying
	\begin{enumerate}[(a)]
		\item $\proj_{T(\Z)}(\M^\top\Q) = \U_{Z}\V_{Z}^\top$
		\item $\norm{\proj_{T(\Z)^\bot}(\M^\top\Q)} < 1$
		\item $\badcolproj(\Q) = \lambda\mc{B}(\S)$
		\item $\norm{\badcolprojc(\Q)}_{2,\infty} < \lambda.$
	\end{enumerate}
	then $(\Z, \S)$ is a solution to Eq.~\ref{eq:sub-lrr}.
	If, in addition, $\badcolproj(\Z^+\Z) = \mbf 0$, and
	$(\Z, \S)$ satisfy the oracle constraints of Eq.~\ref{eq:oracle-constraints},
	then all solutions to Eq.~\ref{eq:sub-lrr} satisfy the oracle constraints
	of Eq.~\ref{eq:oracle-constraints}.
\end{theorem}
\begin{proof}
The proof of this theorem is identical to that of \cite[Thm.~3]{LiuXuYa11} which establishes the same result
when the observation $\C$ is replaced by $\M$.
\end{proof}
It remains to construct a feasible pair $(\Z,\S)$ satisfying the oracle constraints and $\badcolproj(\Z^+\Z) = \mbf 0$ and a \emph{dual certificate} $\Q$ satisfying the conditions of
Thm.~\ref{thm:optimality}.

\subsection{Constructing a Dual Certificate}
To this end, we consider the \emph{oracle problem}:
\begin{align} \label{eq:lrr-oracle}
&\min_{\Z,\S} \quad \norm{\Z}_*+\lambda\norm{\S}_{2,1}\quad \\ \notag
&\text{subject\, to} \\ & \C = \M\Z + \S, \quad \P_{L_{0}^\top}\Z = \Z, \quad\text{and}\quad\badcolproj(\S) = \S. \notag
\end{align}
Let $\Y$ be the binary matrix that selects the columns of $\C$ from $\M$.
Then $(\P_{L_{0}^\top}\Y, \S_{0,i})$ is feasible for this problem, and hence an optimal solution $(\Z^*, \S^*)$  must exist.
By explicitly constructing a dual certificate $\Q$, we will show that $(\Z^*, \S^*)$
also solves the LRR subproblem of Eq.~\ref{eq:sub-lrr}.

We will need a variety of lemmas paralleling those developed in  \cite{LiuXuYa11}.
Let $$\bar\V \defeq \V_{Z^*}\U_{Z^*}^\top\V_{L_0}.$$
The following lemma was established in \cite{LiuXuYa11}.
\begin{lemma}[Lem.~8 of \cite{LiuXuYa11}] \label{lem:11}
	$\bar\V\bar\V^\top = \V_{Z^*}\V_{Z^*}^\top$.  
	Moreover, for any $\A\in\reals^{m \times l}$,
	$$\proj_{T(\Z^*)}(\A) = \P_{L_0^\top}\A + \A\P_{\bar\V} - \P_{L_0^\top}\A\P_{\bar\V}.$$
\end{lemma}

The next lemma parallels Lem.~9 of \cite{LiuXuYa11}.
\begin{lemma} \label{lem:9}
	Let $\hat \H = \mc{B}(\S^*)$.
	Then $$\V_{L_0}\badcolproj(\bar\V^\top) = \lambda \P_{L_0^\top}\M^\top\hat\H.$$
\end{lemma}
\begin{proof}
	The proof is identical to that of Lem.~9 of \cite{LiuXuYa11}.
\end{proof}

Define $$\G \defeq \badcolproj(\bar\V^\top)(\badcolproj(\bar\V^\top))^\top \quad\text{and}\quad \psi\defeq \norm{\G}.$$
The next lemma parallels Lem.~10 of \cite{LiuXuYa11}.
\begin{lemma} \label{lem:10}
	$\psi \leq \lambda^2\norm{\M}^2\gamma l$.
\end{lemma}
\begin{proof}
	The proof is identical to that of Lem.~10 of \cite{LiuXuYa11}, save for the size of $\badcols$, which is now bounded by $\gamma l$.
\end{proof}
Note that under the assumption $\lambda \leq 3/(7\norm{\M}\sqrt{\gamma l})$, we have $\psi \leq 1/4$.

The next lemma was established in \cite{LiuXuYa11}.
\begin{lemma}[Lem.~11 of \cite{LiuXuYa11}] \label{lem:11}
	If $\psi < 1$, then $\badcolproj((\Z^*)^+\Z^*) = \badcolproj(\P_{\bar V}) = \mbf 0$.
\end{lemma}

Lem.~12 of \cite{LiuXuYa11} is unchanged in our setting.
The next lemma parallels Lem.~13 of \cite{LiuXuYa11}.
\begin{lemma} \label{lem:13}
	$\norm{\badcolprojc(\bar\V^\top)}_{2,\infty} \leq \sqrt{\frac{\mu r}{(1-\gamma)l}}$
\end{lemma}
\begin{proof}
	By assumption, $\C = \M\Z^*+\S^*$, $\rank{\C_0} = r$,
	and $\badcolprojc(\C) = \C_0 = \badcolprojc(\C_0)$.
	Hence, $\C_0 = \badcolprojc(\C_0) = \M\badcolprojc(\Z^*)$, and thus
	$$\V_{C_0}^\top = \badcolprojc(\V_{C_0}^\top) = 
	\mSigma_{C_0}^{-1}\U_{C_0}^\top\M\U_{Z^*}\mSigma_{Z^*}\badcolprojc(\V_{Z^*}^\top).$$
	This relationship implies that 
	$$r = \rank{\V_{C_0}^\top} \leq \rank{\badcolprojc(\V_{Z^*}^\top)} \leq \rank{\V_{Z^*}^\top} = r$$
	and therefore that $\badcolprojc(\V_{Z^*}^\top)$ is of full row rank.
	The remainder of the proof is identical to that of Lem.~13 of \cite{LiuXuYa11},
	save for the coherence factor of $(1-\gamma)l$ in place of $(1-\gamma)n$.
\end{proof}

With these lemmas in hand, we define
\begin{align*}
\Q_1 &\defeq \lambda \P_{L_0^\top}\M^\top\hat\H = \V_{L_0}\badcolproj(\bar\V^\top) \\
\Q_2 &\defeq \lambda \P_{(L_0^\top)^\bot}\badcolprojc((\I + \sum_{i=1}^\infty(\P_{\bar V}\badcolproj\P_{\bar V})^i )\P_{\bar V})\M\hat\H\P_{\bar V} \\
&= \lambda\badcolprojc((\I + \sum_{i=1}^\infty(\P_{\bar V}\badcolproj\P_{\bar V})^i )\P_{\bar V})\P_{(L_0^\top)^\bot}\M\hat\H\P_{\bar V},
\end{align*}
where the first relation follows from Lem.~\ref{lem:9}.
Our final theorem parallels Thm.~4 of \cite{LiuXuYa11}.
\begin{theorem} \label{thm:certificate}
	Assume $\psi < 1$, and let
	$$\Q \defeq (\M^+)^\top(\V_{L_0}\bar\V^\top + \lambda\M^\top\hat\H - \Q_1 - \Q_2).$$
	If
	$$\frac{\gamma}{1-\gamma} < \frac{\beta^2(1-\psi)^2}{(3-\psi+\beta)^2\mu r}\,,$$
	$$\frac{(1-\psi) \sqrt{\frac{\mu r}{1-\gamma}}}{\norm{\M}\sqrt{l}(\beta(1-\psi)-(1+\beta)\sqrt{\frac{\gamma}{1-\gamma}\mu r}) }
	< \lambda\,,$$
  and
	$$\lambda < \frac{1-\psi}{\norm{\M}\sqrt{\gamma l}(2-\psi)},$$
	then $\Q$ satisfies the conditions in Thm.~\ref{thm:optimality}.
\end{theorem}
\begin{proof}
	The proof of property $\mbf{S3}$ requires a small modification.
	Thm.~4 of \cite{LiuXuYa11} establishes that
	$\proj_{\mc{I}_0}(\Q) = \lambda\P_{M}\hat\H$.
	To conclude that $\proj_{\mc{I}_0}(\Q) = \lambda\hat\H$, we note that
	$\S_{i}^* = \C - \M\Z^*$ and that the column space of $\C$ contains
	the column space of $\M$ by assumption.
	Hence, $\P_{M}\S_{i}^* = \S_{i}^*$ and therefore $\proj_{\mc{I}_0}(\Q) = \lambda\P_{M}\hat\H = \lambda\hat\H$.
	
  The proofs of properties $\mbf{S4}$ and $\mbf{S5}$ are unchanged except for
the dimensionality factor which changes from $n$ to $l$.
\end{proof}

Finally, Lem.~14 of  \cite{LiuXuYa11} guarantees that the preconditions of Thm.~\ref{thm:certificate} are met
under our assumptions on $\lambda, \gamma^*,$ and $\gamma$.

\end{document}